\documentclass{article}

\usepackage{microtype}
\usepackage{graphicx}
\usepackage{booktabs} 

\usepackage[OT1]{fontenc}
\usepackage{textcomp}
\usepackage{booktabs} 
\usepackage{comment}
\usepackage{wrapfig}
\usepackage{xcolor}
\usepackage{graphicx}
\usepackage{subfig}
\usepackage{array}
\usepackage{multicol}
\usepackage{afterpage}
\usepackage{float}
\usepackage{listings}
\usepackage{multirow}
\usepackage{arydshln}
\usepackage{xspace}
\usepackage{enumitem}
\usepackage{fancyhdr}
\usepackage{amsmath}
\usepackage{url}
\usepackage[htt]{hyphenat}
\usepackage{subfloat}
\usepackage{amsthm}
\usepackage{multirow}
\usepackage{amsfonts}
\usepackage{makecell}

\usepackage[preprint]{icml2026}

\usepackage[capitalize,noabbrev]{cleveref}

\theoremstyle{plain}
\newtheorem{theorem}{Theorem}[section]

\theoremstyle{definition}

\theoremstyle{remark}

\usepackage[textsize=tiny]{todonotes}

\usepackage{xspace}

\usepackage{pifont}
\usepackage{xcolor}

\newcommand{\phm}[1]{\vspace{.4em} \noindent\textbf{#1}\hspace{.5em}}

\icmltitlerunning{\textit{Justitia}: Fair and Efficient Scheduling of Task-parallel LLM Agents with Selective Pampering}

\begin{document}

\twocolumn[
\icmltitle{\textit{Justitia}: Fair and Efficient Scheduling of Task-parallel LLM Agents\\ with Selective Pampering}

\icmlsetsymbol{equal}{*}
\begin{icmlauthorlist}
\icmlauthor{Mingyan Yang}{equal,yyy}
\icmlauthor{Guanjie Wang}{equal,yyy}
\icmlauthor{Manqi Luo}{yyy}
\icmlauthor{Yifei Liu}{yyy}
\icmlauthor{Chen Chen}{yyy}
\icmlauthor{Han Zhao}{yyy}
\icmlauthor{Yu Feng}{yyy}
\icmlauthor{Quan Chen}{yyy}
\icmlauthor{Minyi Guo}{yyy}
\end{icmlauthorlist}
\icmlaffiliation{yyy}{Shanghai Jiao Tong University}
\icmlcorrespondingauthor{Chen
Chen}{chen-chen@sjtu.edu.cn}
\icmlkeywords{Machine Learning, ICML}
\vskip 0.3in
]



\printAffiliationsAndNotice{\icmlEqualContribution} 

\begin{abstract}
	LLM agents, which often comprise parallel inference tasks, are commonly adopted to solve real-world problems. 
	When serving such task-parallel LLM agents in shared GPU servers, the scheduler is expected to attain fast agent completion with guaranteed worst-case performance.
	For that objective, our insight is to selectively pampering agents based on their completion order under idealized fair-sharing. 
	We design \textit{Justitia}, a fair and also efficient scheduler for task-parallel LLM agents. 
	Noticing that memory is prevalently a bottleneck in LLM serving, \textit{Justitia} quantifies the true agent cost in a memory-centric manner. 
	It also adopts a light-weight yet accurate method to predict agent costs.
        Finally, \textit{Justitia} adopts a virtual-time based fair queuing algorithm to reduce the overall performance with guaranteed worst-case delay. We have implemented \textit{Justitia} atop vLLM, and experimental results involving diverse agents show that it can substantially enhance the scheduling efficiency with fairness preserved.

\end{abstract}

\section{Introduction}

\label{sec:intro}


The recent years have witnessed the boom of large language models (LLMs)~\cite{naveed2023comprehensive} in revolutionizing various fields~\cite{thirunavukarasu2023large,gu2023llm,luo2025large}.
In particular, 
it has become increasingly popular to address realistic problems with \emph{LLM agents}, which usually contain a set of parallel inference tasks.
For example, to summarize a very large file, multiple parallel inferences would be launched, each processing one file partition~\cite{lin2024parrot}; besides, to solve a complex mathematical problem, multiple searching directions need to be explored concurrently, until a satisfiable answer is obtained~\cite{fu2024efficiently}.
Moreover, such LLM agents are prevalently served in \emph{shared GPU servers}, where users usually expect to get the final output as early as possible~\cite{lin2024parrot,tan2025towards}---without being remarkably delayed by competitors~\cite{vtc,fairserve}.

However, existing LLM serving systems fail to behave well in both efficiency and fairness for task-parallel LLM agents.
For example, mainstream LLM frameworks~\cite{vllm,zheng2024sglang} schedule LLM inference requests in a \emph{First-Come-First-Serve} (FCFS) manner, being inefficient due to \emph{head-of-line-blocking}.
Meanwhile, while prioritizing short jobs~\cite{qiu2024efficient,shahout2024don} can improve efficiency, it however incurs starvation. 
To ensure fairness, the VTC scheduler~\cite{vtc} is proposed to fairly allocate the serving resources among different users.
However, this method forces each agent to only use its fair share, resulting in postponed completion. 

To attain efficient and also fair serving, our insight is that LLM agents shall be served in a {saturated} manner following the {fair completion order}, which we call \emph{selective agent pampering}.
We first note that, regarding fairness, users primarily care about the \emph{long-term} fairness (i.e., with guaranteed  fast completion) instead of \emph{short-term} fairness (i.e., always allocated an equal share of the total resources). 
In that sense, as shown in Fig.~\ref{fig:toy_example}, when two agents are competing with each other, serving them---not \emph{concurrently with the fair share}, but \emph{sequentially in the fair completion order with all the resources}---can reduce the average completion time, with no agent actually delayed.
\begin{figure}
    \centering
    \subfloat[Fair Sharing]{
        \centering
        \includegraphics[width=0.35\linewidth]{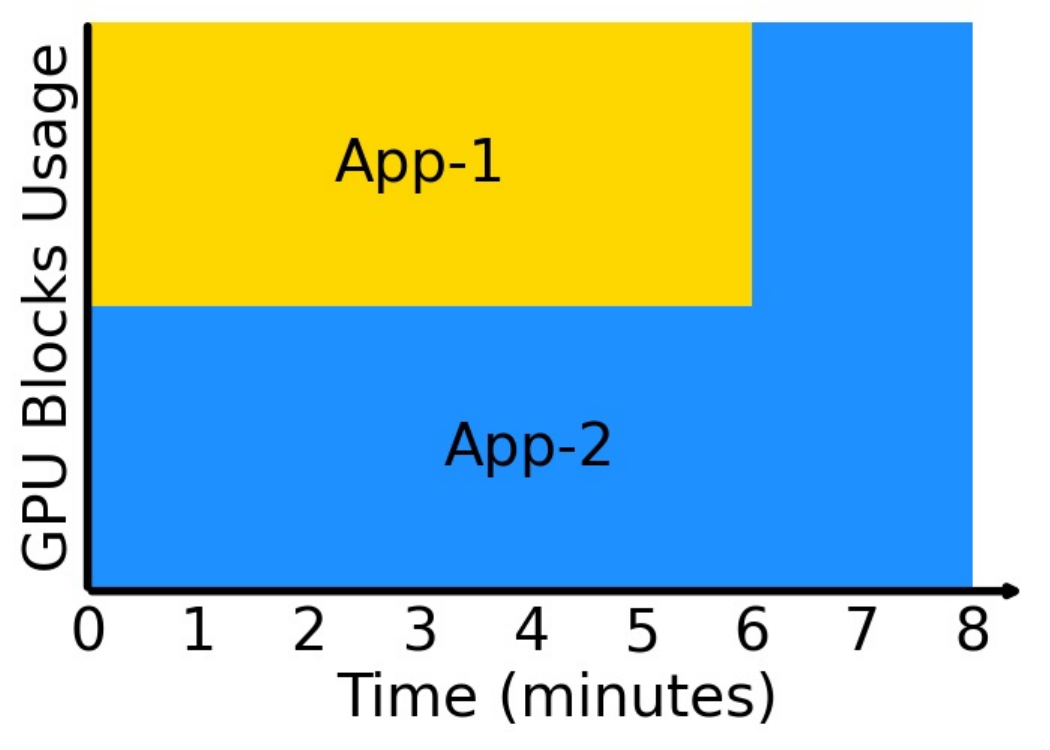}
    }\hspace{0.2in}
    \subfloat[Selective Pampering]{
        \centering
        \includegraphics[width=0.35\linewidth]{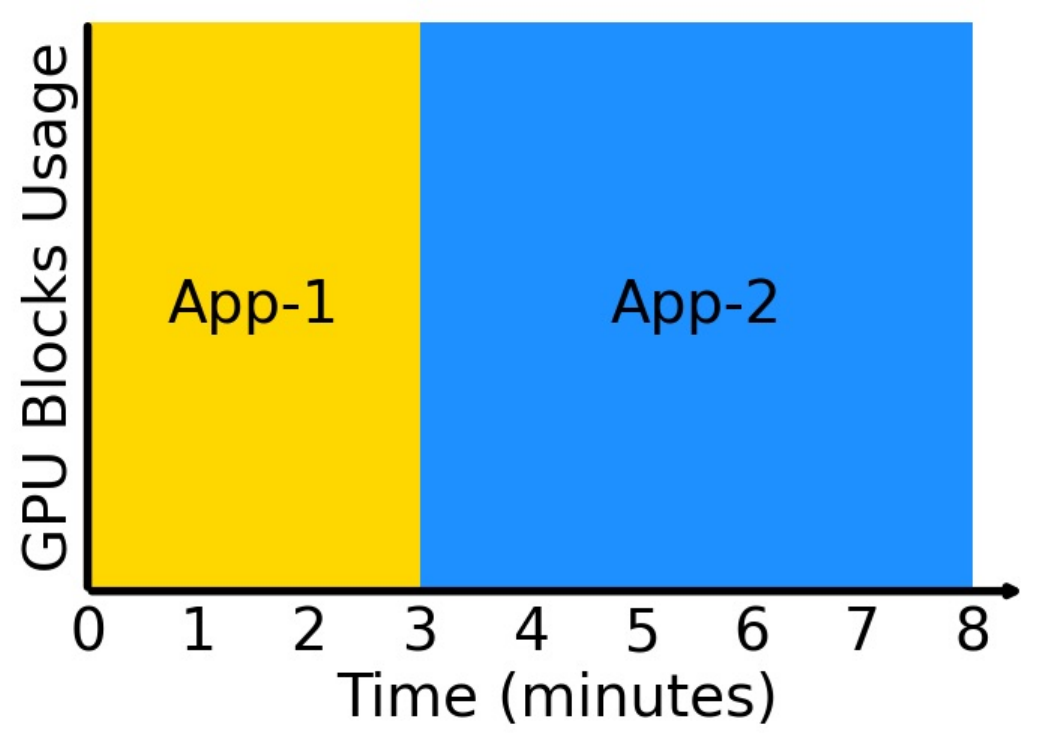}
    }
    \caption{Given two competing LLM agents, serving them sequentially following their fair completion order, or \emph{selective pampering}, can reduce the average completion time with no per-agent delay.} 
    \label{fig:toy_example}
    \vspace{-.1in}
\end{figure}

Yet, it is a non-trivial task to apply the above insight into practice, which requires 
estimating the service cost of LLM agents and 
also calculating their fair completion order.
Specifically, LLM inferences consume both compute and memory (i.e., KV cache) resources on GPU servers, rendering it difficult to precisely quantify an agent's true service cost.
Meanwhile, 
the overall service cost of an LLM agent needs to be predicted upon its arrival; such prediction shall be made accurate and also efficient.
Moreover, upon the arrival or completion of any agent, we need to work out the expected completion order of all the pending agents under a fair scheduler; such decisions shall be made efficiently (without frequent refreshing later) and also of high quality (yielding fast completion with worst-case guarantee).


In this paper, we propose \emph{\textit{Justitia}}, a fair and efficient scheduler for task-parallel LLM agents.
In modeling the service cost of an LLM agent, we find that the GPU memory is usually the true bottleneck; therefore, rather than adopting compute-centric cost modeling as in \cite{vtc}, we propose \emph{memory-centric cost modeling}, jointly considering both \emph{space} and \emph{time} of memory occupation.
Meanwhile, in \textit{Justitia} we adopt a Multi-Level Perception model for agent cost prediction, which is light-weight (in terms of the training and inference cost) and also accurate (by leveraging agent-specific demand stability). 
In determining the agent queuing order, we learn from the virtual-time based \emph{fair queuing} algorithm originally for network packet scheduling~\cite{demers1989analysis,parekh1993generalized}.
That algorithm can efficiently calculate the agent completion order under an idealized fair scheduler---in one shot with no need to consider later-arrival ones.
Moreover, we theoretically prove that, the 
maximal delay an agent could encounter under \textit{Justitia}---when compared to under the idealized fair scheduler---is always bounded by a constant.

We have implemented \textit{Justitia} above vLLM~\cite{vllm}, a mainstream LLM serving framework. 
Given a set of diverse LLM agents,  \textit{Justitia} can remarkably improve the average completion time without remarkably delaying any agent. 
Specifically, compare to the state-of-the-art fair scheduler, VTC, our \textit{Justitia} scheduler can reduce the average agent completion time by 57.5\%---with only 8\% agents delayed (their average delay scale is less than 10\%).
Moreover, further studies confirm that it is indeed indispensable to adopt memory-centric service cost modeling and MLP-based demand prediction for LLM agents; besides, the scheduling overhead of \textit{Justitia} is also negligible.
\section{Background}
\label{sec:background}


\subsection{Task-parallel LLM Agents} 
\label{subsec:basics} 

\begin{figure*}
		\centering
		\subfloat[MapReduce Summary]
		{
			\centering
			\includegraphics[height=0.19\linewidth]{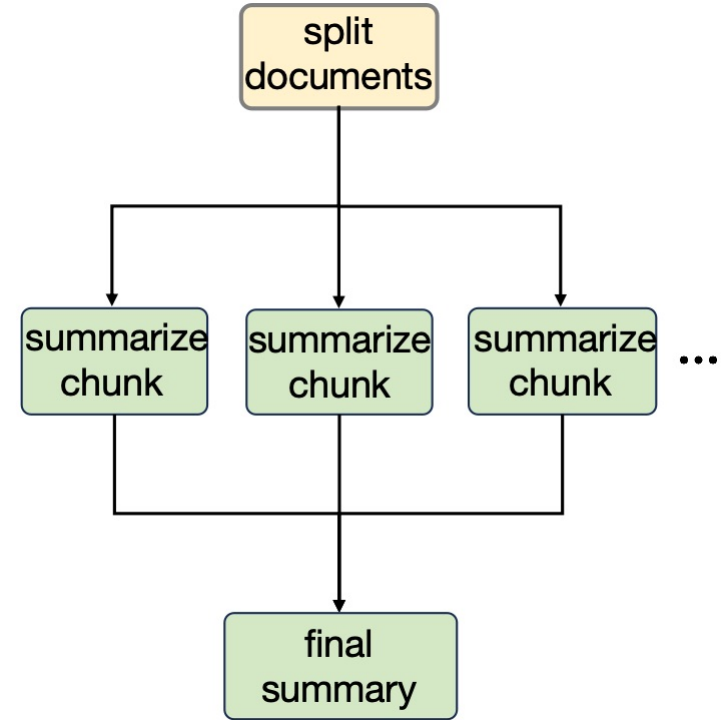}
			\label{fig:app:map_reduce}
		}\hfill
		\subfloat[Doc Merging]
		{
			\centering
			\includegraphics[height=0.19\linewidth]{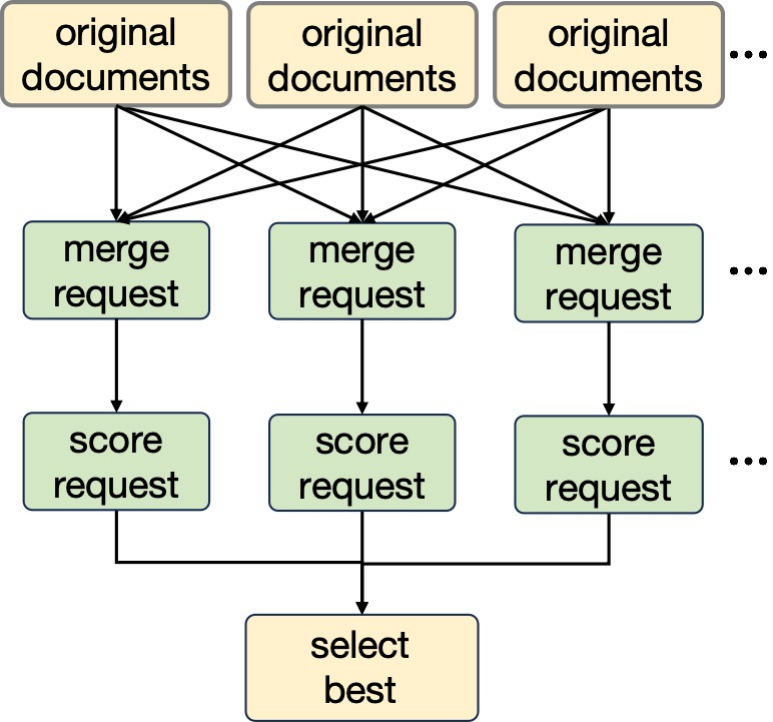}
			\label{fig:app:doc_merge}
		}\hfill
		\subfloat[Fact Verification]
		{	
			\centering
			\includegraphics[height=0.19\linewidth]{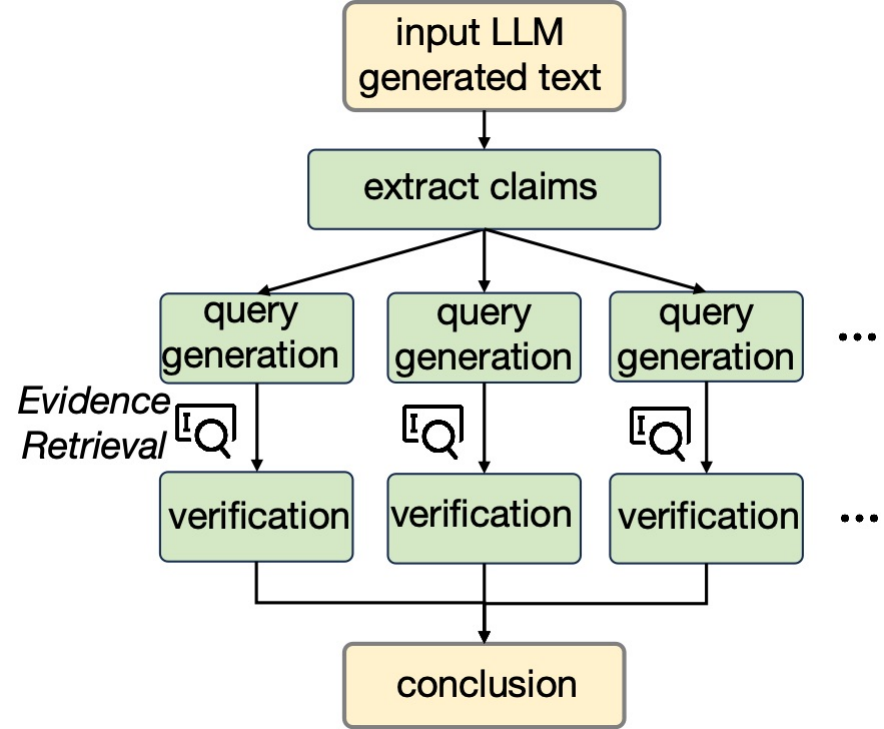}
			\label{fig:app:fact_verification}
		}\hfill
		\subfloat[Self Consistency]
		{	
			\centering
			\includegraphics[height=0.19\linewidth]{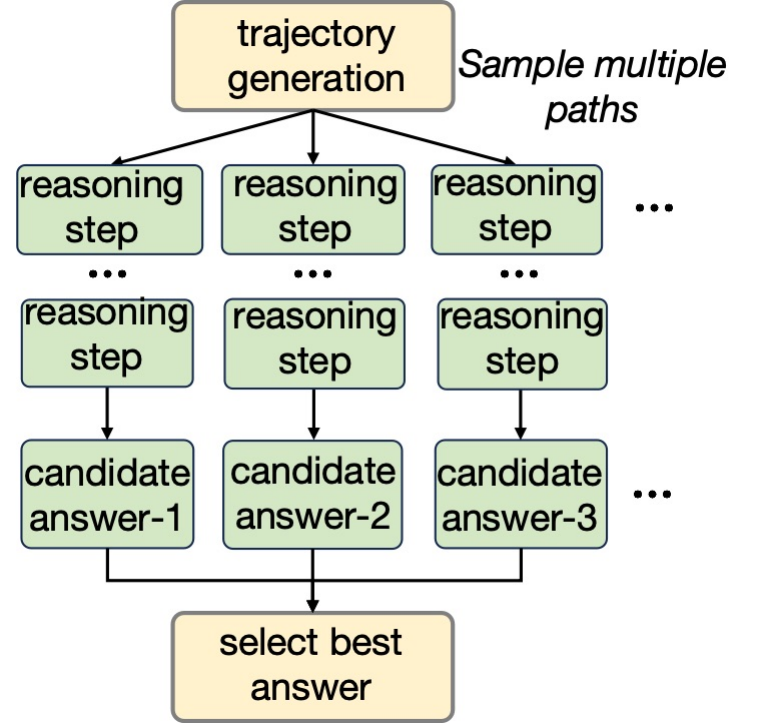}
			\label{fig:app:sc}
		}
                \vspace{.05in}
		\caption{Examples of task-parallel LLM agents.}
		\label{fig:llm_apps}
        \vspace{.05in}
\end{figure*}


\phm{The emergence of task-parallel LLM agents.}
Large language models (LLMs) have demonstrated strong capability for diverse tasks like text processing~\cite{wu2025survey}, code generation~\cite{gu2023llm} and workflow planning~\cite{kannan2024smart}. 
To apply LLMs with enhanced capability, there recently emerges a trend to have multiple LLM inferences jointly solve real-world problems---in the form of \emph{LLM agents}~\cite{lin2024parrot,luo2025large}.

For example, as shown in Fig.~\ref{fig:app:map_reduce}, since the context window size of an LLM inference is limited (e.g., 10M tokens~\cite{gemini}), to summarize a large text file, multiple parallel inferences---each processing a slice of the original file---need to be launched and finally have their results merged~\cite{lin2024parrot}. 
Similarly, in Fig.~\ref{fig:app:doc_merge}, to improve the merging quality of multiple documents, several merging inferences are submitted in parallel---each followed by a scoring inference~\cite{besta2024graph}.
In Fig.~\ref{fig:app:fact_verification}, to address the \emph{hallucination} problem~\cite{ji2023survey}, each claim within the LLM-generated output needs to be additionally verified by a dedicated LLM inference (e.g., with the \emph{FacTool} framework~\cite{chern2023factool}).
In Fig.~\ref{fig:app:sc}, to solve complex mathematical problems, the self-consistency (SC) agent~\cite{wang2022self} expands multiple reasoning trajectories and applyies majority voting to find the best answer.

Note that, each agent above comprises a set of \emph{parallel inference tasks}; we call such agents \emph{task-parallel LLM agents}.
Task-parallel LLM agents would become an increasingly common workload paradigm in future LLM serving clusters.
They are thus the focus of this paper. 

\phm{Scheduling objective for task-parallel LLM agents.}
For each LLM agent above, the final output desired by the the end-user is only available when all the inference tasks of that agent complete.
Therefore, when serving those LLM agents in a GPU cluster, it is of paramount significance to reduce their \emph{end-to-end execution time}~\cite{lin2024parrot}.
In the meantime, since different agents---usually submitted by different users---would compete for the limited processing capability on GPU servers,
it is also necessary to ensure service fairness among LLM agents, so as to avoid negative interferences~\cite{vtc} (e.g., \emph{head-of-line-blocking} or \emph{starvation}).
In summary, \emph{efficiency} and \emph{fairness} are two critical scheduling objectives for LLM agents. 





\subsection{Related Works and Their Limitations}
\label{sec:closely_related}
Regarding LLM inference serving, a series of scheduling methods have been proposed recently.
Here we revisit such scheduling practices and summarize their limitations.

Given that the output length of an LLM inference is non-deterministic, mainstream LLM serving frameworks like vLLM~\cite{vllm} and SGLang~\cite{zheng2024sglang} commonly adopt the \emph{First-Come-First-Serve} (FCFS) algorithm in request scheduling.
However, they suffer the head-of-line-blocking problem, meaning that a large task would impede the execution of those behind it.
Such a problem also appears for agent-level schedulers like Parrot~\cite{lin2024parrot} and Ayo~\cite{tan2025towards}, which also apply FCFS.

Given the inefficiency of FCFS schedulers, some other schedulers are designed to be \emph{efficiency-centric}. 
For example, to tackle head-of-line-blocking, FastServe~\cite{wu2023fast} employs a \emph{Multi-Level-Feedback-Queue} (MLFQ) for scheduling, yet it still compromises the end-to-end execution latency due to the periodical execution mode. 
Recently, some prediction-based methods propose to approximate \emph{Shortest-Job-First} (SJF) in scheduling~\cite{fu2024efficient,qiu2024efficient,shahout2024don}, yielding near-optimal efficiency.
However, such efficiency-centric schedulers may starve large agents, failing to provide service guarantee. 


To guarantee the worst-case performance, some \emph{fairness-centric} schedulers are also proposed.
For example, Sheng~\emph{et~al.} proposed a fair scheduler called \emph{Virtual Token Counter} (VTC)~\cite{vtc}, which tracks the services received for each tenant and prioritize the ones receiving the least services.
Another work, \emph{FairServe}~\cite{fairserve}, uses a similar scheduling method to avoid cross-tenant interference. 
However, by enforcing each task-parallel agent to only use its fair share, the end-to-end execution time---which users truly care about---would be delayed compared to the cases with monopolized resources. 

To summarize, existing LLM scheduling methods fail to attain efficient and fair scheduling at the agent level, compromising the user experience.
Our objective is thus to make up that research gap by designing a fair and also efficient scheduler for task-parallel LLM agents. 
On the one hand, we want to optimize agent-level scheduling efficiency by reducing the average completion time; on the other hand, we want to provide service guarantee on the worst-case service degradation any agent could experience compared to the case where it is allocated an equal resource share. 




\section{Motivation}
\label{sec:motivation}

\subsection{Insight} 
\label{subsec:insight}

\phm{Instantaneous fairness or finish-time fairness?}
While both efficiency and fairness are appealing objectives, they are however conflicting with each other.
We note that the fairness-centric schedulers~\cite{vtc,fairserve} are inefficient essentially because they seek to fairly share the resources at each instant moment, i.e., stick to \emph{instantaneous fairness}.
However, in expecting service fairness, users primarily desire a service guarantee on the worst-case completion time at the agent level.
In that sense, long-term fairness (i.e., guaranteeing agent completion time) would suffice for users; a typical fairness definition in that regard is \emph{finish-time fairness}~\cite{mahajan2020themis}, which relates to \emph{the ratio of the running time in a shared cluster with $N$ tenants to running alone in a $\frac{1}{N}$ cluster}.
We thus materialize out fairness objective as long-term finish-time fairness, which, as we show next, allows for more flexibility to balance between both sides.


\phm{Balance between efficiency and fairness with \emph{selective pampering}.}
For a set of task-parallel LLM agents, 
instantaneous fair sharing would limit the usable resources of each agent to only the average share.
Instead, we can prioritize the agent one by one based on \emph{their relative completion order under fair sharing}, allowing each prioritized agent to use unlimited resources it desires---we call this principle \emph{selective pampering}.
In this way, because a prioritized agent can complete faster, the average agent completion time can be reduced.
Meanwhile, this method can also behave well in the \emph{fairness} aspect:
since the pampered agents would yield their resources also earlier, and the fair-completion-order can prevent head-of-line blocking or infinite starvation, 
each agent may still finish no later than its expected completion time under fair sharing. 
In summary, with selective pampering we can do well in both efficiency and fairness.
		
To confirm, we conduct a testbed experiment as shown in Fig.~\ref{fig:testbed_example}. 
Note that to avoid repeated computations, the intermediate feature states of the generated tokens are cached for later usage during the inference process (called KV cache)~\cite{gao2025fast}, thus the KV cache space is a key resource contended on GPUs.
We submit two DocMerging agents simultaneously to an LLaMA2-7B model deployed on a single A100 GPU, with a total KV cache block number of 459.
As shown in Fig.~\ref{fig:fair}, under instantaneous fair sharing, each agent is restricted to use its fair share (as indicated by the KV block usage), resulting in an average \emph{job (agent) completion time}, or JCT, of 210 s.
However, if we prioritize the agents sequentially (based on the job completion order in Fig.~\ref{fig:sequential}), the average JCT drops to 166 s
---without delaying any single agent compared to the fair-sharing case in Fig.~\ref{fig:fair}.
This verifies the effectiveness of \emph{selective pampering}. 

\begin{figure}
    \centering
    \subfloat[Instantaneous Fair Sharing]{
        \label{fig:fair}
        \centering
        \includegraphics[width=0.42\linewidth]{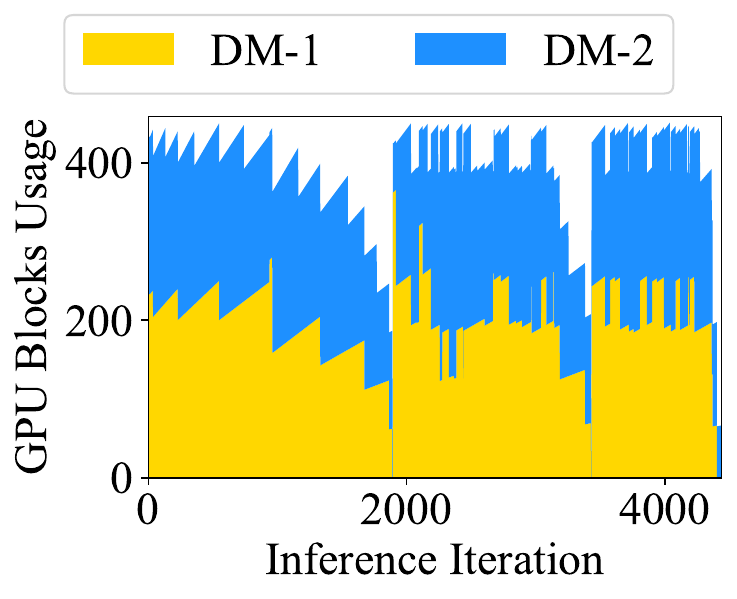}
    }
    \hfill
    \subfloat[Pampering in Fair Order]{
        \label{fig:sequential}
        \centering
        \includegraphics[width=0.42\linewidth]{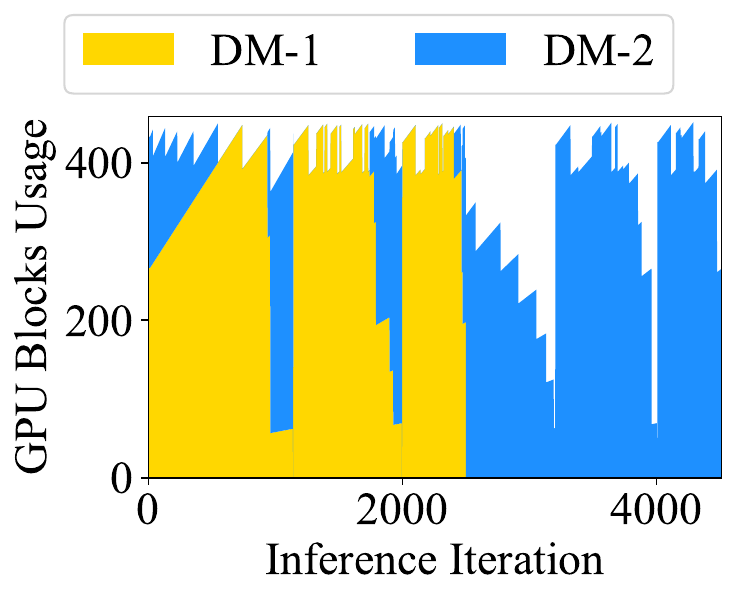}
    }
    \caption{KV block usage when running two DocMerging (DM) agents under different scheduling schemes.} 
    \label{fig:testbed_example}
\end{figure}

\subsection{Challenges}
\label{subsec:challenges}

To apply \emph{selective pampering}, we need to obtain the agent completion order prior to agent execution.
However, realizing it in practice confronts two key challenges.

First, it is difficult to quantitatively get the service cost of an LLM agent upon its arrival.
On the one hand, serving LLM workloads requires both compute and memory resources; simply using the compute amount (e.g., total output token length) often fails to quantify the true service cost of an LLM agent.
On the other hand, due to the auto-regressive nature of LLM inferences, the overall resource demands of an LLM agent are highly uncertain a priori.

Second, given the estimated agent-serving cost, it is difficult to determine the agent queuing order in an \emph{efficient} (in terms of the decision speed) yet also \emph{high-quality} (in terms of the ultimate performance) manner.
Enforcing \emph{selective pampering} requires calculating agent completion time under a fair scheduler, which is hard to make given the dynamic arrival of later agents. 
In particular, it is desirable to provide a theoretical guarantee depicting the worst-case impact of \emph{selective pampering} on the performance of any agent.
\section{Solution}
\label{sec:solution}
In this section, we present \emph{\textit{Justitia}}, a \emph{fair} and also \emph{efficient} scheduler for task-parallel LLM agents, which applies the \emph{selective pampering} principle with all the above challenges addressed. 
We first introduce our cost modeling method in Sec.~\ref{subsec:modeling}, and then elaborate how to predict the agent-level resource demand in Sec.~\ref{subsec:prediction}. 
Finally we describe the queuing strategy of \textit{Justitia} in Sec.~\ref{subsec:queuing}. 


\subsection{Memory-Centric Cost Modeling}
\label{subsec:modeling}
To determine the agent scheduling order, we need to quantify the overall serving cost of each LLM agent, and the key is to describe the serving cost of an LLM inference. 


\begin{figure}
	\centering
	\includegraphics[width=0.14\textwidth]{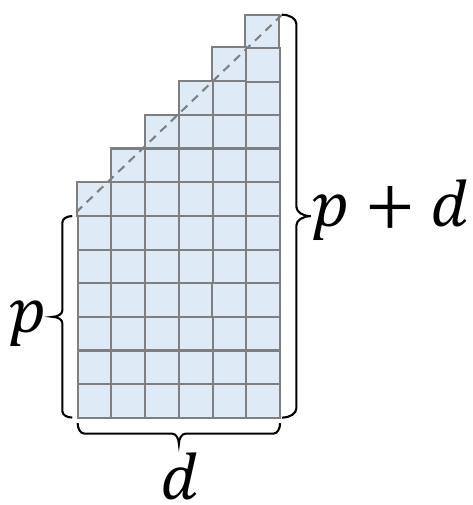} 
	\caption{Demand modeling of an LLM inference. $p$ and $d$ are respectively the prefill and decoding token length. The service cost is then depicted as the accumulative KV cache occupation (i.e., KV token-time): $pd +\frac{d^2}{2}$.} 
    \label{fig:quantification_example}
    \vspace{-.1in}
\end{figure}

Given that both computation and memory resources are demanded in LLM inference process, its serving cost can be captured from either the \emph{computation} perspective or the \emph{memory} perspective. 
The VTC work~\cite{vtc} chooses the former:
it measures the serving cost as a weighted combination of the prefilling (input) tokens and decoding (output) tokens.
However, such a method is over-simplified by 
ignoring the impact of KV cache consumptions.
In fact, in modern LLM serving frameworks like vLLM~\cite{vllm}, the inference throughput is bounded commonly on the GPU memory (\emph{KV cache space}): new sequences can always be added to the \emph{running} queue as long as there are sufficient KV cache space; otherwise some running sequences would be interrupted and moved to the \emph{swapped} queue.
In that sense, inference serving cost shall be better depicted from the \emph{memory} perspective. 

Moreover, the resource demand of an LLM inference is not fixed during its lifetime. 
During the generative inference process, the sequence length keeps increasing, leading to increasing KV-cache occupations. 
Therefore, in this paper we propose a \emph{memory-centric} cost modeling method that depicts the cumulative service cost of an LLM inference in both temporal and spatial dimensions. 
To be specific, let $p$ and $d$ respectively represent the prefill and decode token length, then we devise the cost metric, \emph{KV token-time}, as:
\begin{equation}
    c=\sum_{i=1}^{d}(p+i)=pd +\frac{d^2}{2}.
\end{equation}
This formula adds up the KV cache occupation\footnote{
The unit of the KV cache occupation is the number of KV cache blocks corresponded to one token (over all the LLM layers and heads), which is fixed for a given LLM.
This is more concise for analysis than recording the exact KV block number.
We describe the total KV cache space also in such units.
} 
over all the iterations\footnote{
For simplicity we do not consider the time difference when generating different tokens. 
In realistic execution, an LLM sequence would be batched with sequences from other inference requests~\cite{yu2022orca}; all such sequences jointly determine the per-iteration inference latency.
The inference time of such runtime batches with mixed sequence is statistically stable.
} of that inference.  
From the above formula, we can learn that the relationship between the cost volume and the generation length ($d$) is indeed \emph{quadratic}. 
This is more reasonable compared to the \emph{linear} relationship in the cost model of VTC~\cite{vtc} given the inflation effect of $d$ on both memory occupation size and duration.

Furthermore, the overall serving cost of an LLM agent can be naturally defined as the sum of the KV token-time of all its constituting inferences. 
Such a memory-centric modeling method can express the true serving cost of LLM agents.

\subsection{MLP-based Demand Prediction}
\label{subsec:prediction}
\begin{figure}[]
	\centering
	\includegraphics[width=0.45\textwidth]{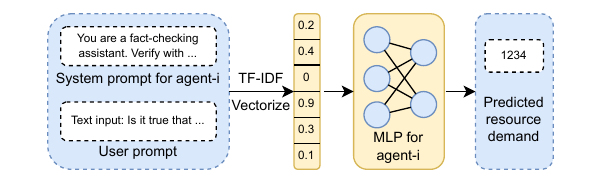} 
	\caption{Service cost prediction for agents with an MLP model.} 
	\vspace{-0.1in}
\label{fig:predict}
\end{figure}


Given the above cost modeling method, to support our previous insight of \emph{selective pampering}, we need to predict the cost volume of each agent. 
There are already some attempts~\cite{zheng2024response,s3} for demand prediction of LLM workloads. 
For example, the $S^3$~\cite{s3} method proposes to fine-tune a language model, \emph{Distillbert}, to predict the output length range given the inference prompt. 
However, simply adopting such a method for agent cost prediction is inappropriate in efficiency and accuracy.
First, fine-tuning the Distillbert model (with 66 million parameters) is time-consuming and also requires collecting a large number of agent execution samples, which is hard to obtain in practice.
Meanwhile, using the Distillbert model for prediction---which is essentially yet another LLM inference---would \emph{incur non-negligible runtime overhead}.
Moreover, the $S^3$-like method uses a single model to predict all the workloads, yet different agents may exhibit heterogeneous cost distribution patterns, rendering single-model prediction \emph{inaccurate}.

Therefore, in this paper we seek to devise a \emph{light-weight} and also \emph{sufficiently-accurate} cost prediction method for LLM agents.
For \emph{high accuracy}, we respectively maintain a prediction model for each agent: our measurements show that the resource demand of an agent is relatively stable across different trial runs (due to space limitation, we put it to Appendix~\ref{appendix:agent_demand_stability}); in that sense, the agent type can play as a valuable \emph{prior knowledge}.
Meanwhile, for \emph{low overhead}, for each agent type we maintain a \emph{Multi-Layer Perceptron} (MLP) model to predict the its overall service cost based on the agent input.
Since such MLP models have a simple structure, they can be efficiently trained even with limited historical data; meanwhile, an MLP model requires minimal computational resources to make predictions, suitable for real-time scheduling tasks. 

Fig.~\ref{fig:predict} elaborates the workflow of our MLP-based cost prediction method.
For the runtime input prompt, we first perform vectorization using \emph{Term Frequency-Inverse Document Frequency} (TF-IDF)~\cite{sparck1972statistical}. 
TF-IDF is a lightweight and efficient method for converting text into numerical vectors, focusing on word importance rather than deep semantic analysis. 
It's ideal for quick processing with minimal overhead. 
Then the vectorized input will be passed into the agent-specific, 4-layer MLP to get the agent cost. 
The number of neurons in the first layer is proportional to the average agent input size.
The training is conducted on 100 samples per agent, optimized via gradient descent with Mean Squared Error (with L2 regularization).
Experimental results later in Sec.~\ref{subsec:ablation} demonstrate that this method achieves high accuracy with low overhead.

\subsection{Fair and Efficient Agent Queuing}
\label{subsec:queuing}

In this part we elaborate the key queuing algorithm of \textit{Justitia}. 
In \textit{Justitia}, each inference is queued based on the overall demand of the LLM agent it belongs to, so that all the inferences of a high-priority agent can be served consecutively without being interleaved. 
In particular, considering the preemption overhead (e.g., KV cache swapping), we follow the original non-preemptive scheduling principle\footnote{
To avoid inference interference, vLLM is designed to be non-preemptive~\cite{vllm}: if the KV cache space is used up, the running inference would be placed into the swapped queue, which has a higher priority; when the KV cache space is available again, the agents in the swapped queue (if any) would be served first. 
In that sense (with the swapped queue viewed as part of the running queue), 
any running inference would never be preempted by those pending requests in the waiting queue.
} in vLLM~\cite{vllm}: any pending request in the waiting queue---regardless of its priority---cannot preempt a running inference; that is, agent-level preemption only occurs after the completion of an entire inference. 

Recall that our insight is to conduct saturated agent serving following the fair completion order; the key is to acquire that fair completion order efficiently---at agent arrival time.
One choice is to use the VTC scheduler~\cite{vtc} as the reference system.
Yet, idealized fair sharing requires each agent stick to its fair share, whereas an inference scheduled in VTC can take an arbitrarily large KV cache space as needed to accommodate its prompt, disrupting the expected execution status.
In that sense, getting the VTC completion order requires a trial run (or simulation) based on the full agent demand knowledge (including those arriving later).
Moreover, estimating the fair completion order requires continuously refreshing---upon the arrival or completion of any agent---the remaining resource demand and the latest resource fair share, which is cumbersome.

Fortunately, 
we find that the problem we face now is similar to the packet scheduling problem in the networking field.
In network scheduling, the network port can only send one packet at a time (which is non-preemptable), yet it is expected that different packets can fair share the network resource at each instant (to attain flow-level fairness). 
In that case, the packet scheduler shall deliberately select the next packet to transmit, with the objective to minimize the delay any packet may encounter compared to the fair-sharing case. 

The classical solution to the packet scheduling problem is \emph{fair queuing}~\cite{demers1989analysis,parekh1993generalized}.
In principle, the fair queuing algorithm also use the completion order of different packets under idealized fair sharing as their scheduling priorities.
That idealized fair sharing scheme is called \emph{Generalized Processor Sharing} (GPS), where the backend resources are arbitrarily divisible and equally allocated to each active tenant.
As our problem, it is hard to estimate the completion time in GPS at packet arrival time---without knowing the runtime packet contention status; to get the relative scheduling order in one shot at packet arrival time, the fair queuing algorithm employs a notion called \emph{virtual time}, which seeks to adjust the time elapsing rate instead of each agent's expected fair completion time when the agent contention state changes.
We thus borrow such virtual-time based fair queuing method for scheduling LLM agents.

\begin{figure}[]
\includegraphics[width=0.4\textwidth]{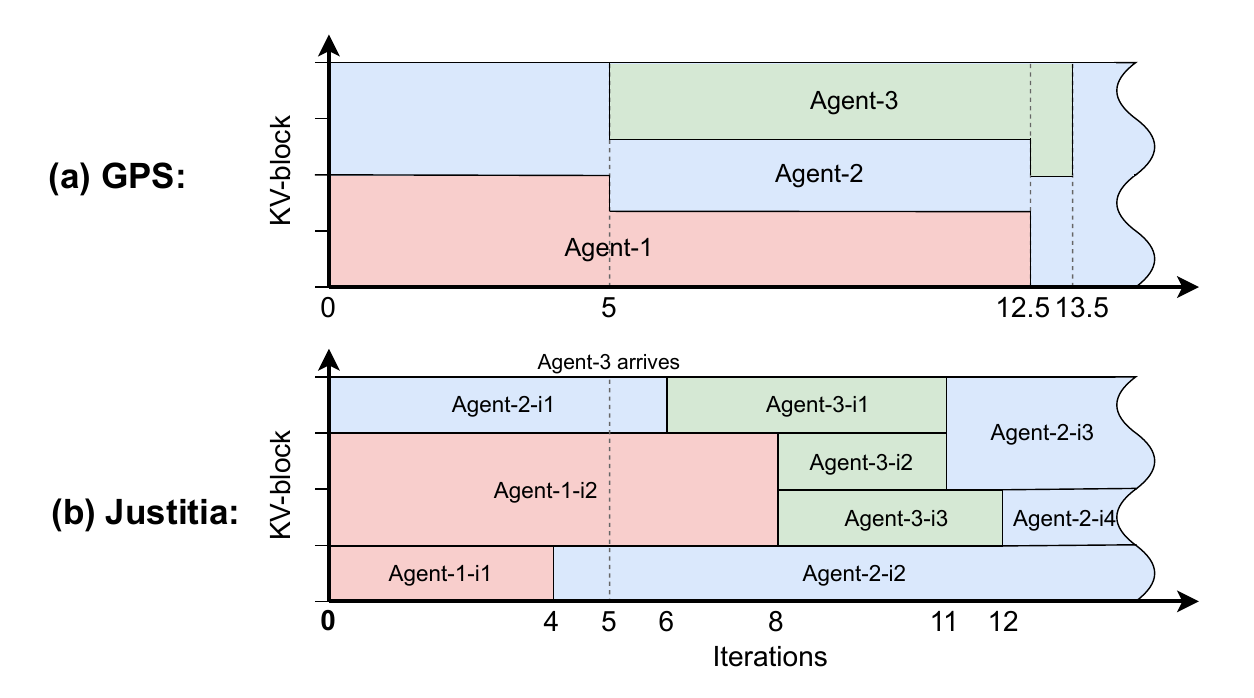}
\vspace{-.1in}
\caption{In \textit{Justitia}, agents are prioritized according to their relative completion order under GPS. 
} 
\label{fig:queuing_strategy}
\vspace{-.1in}
\end{figure}

Fig.~\ref{fig:queuing_strategy} presents the key queuing mechanism of our proposed \textit{Justitia} scheduler (for simplicity we ignore demand volume inflation during inference).
For LLM agent scheduling, we also use the GPS scheme to gauge the fair completion order (i.e., the scheduling priority) of LLM agents.
As shown in Fig.~\ref{fig:queuing_strategy}, suppose the agent completion order under the idealized fair-sharing system, GPS, is Agent-1, Agent-3, and Agent-2, then we set that order as their scheduling priority.
In that sense, when Agent-2-i1 (the first inference request of Agent-2) completes, the requests from Agent-3 would take the service opportunity.
Specifically, for efficient completion order estimation, we adopt the notion of virtual time $V (t)$, which is defined as a function of the real time $t$:
\begin{equation}
  \label{eq:vtime}
  \begin{split}
    V(0) & = 0,\\
    \textstyle
    \frac{\mathrm{d}}{\mathrm{d} t} V(t) & = M / N_t.
  \end{split}
\end{equation}
Here, $M$ is the total KV cache space, and $N_t$ is the number of active agents in GPS at time $t$; $M / N_t$ thus represents the instantaneous fair share.
Further, $V(t)$ increases at the marginal rate at which each agent receives service in GPS.
When Agent-$j$ arrives at time $a_j$ (with ${C}_j$ be its KV token-time cost), \textit{Justitia} calculates its \emph{virtual finish time} $F_j$---the time at which the agent would complete in GPS---as:
\begin{equation}
  \label{eq:v_fin_time}
  \bar{f}_j = V(a_j) + {C}_j.
\end{equation}
The virtual finish time of an agent, once calculated, requires no update in the future.
While the arrival of later agents would change the fair-serving rate, they do not change the relative completion order among existing agents (i.e., the relative order in $\{F_j\}$), because each active agent would always be serviced with the same rate.
\textit{Justitia} further adopts the $\{F_j\}$ order as the scheduling priority.
In this way we can attain low scheduling overhead: the status refreshing overhead on agent arrival or completion is \emph{constant}, and the complexity to select the next agent to serve is $O(\text{log }N_t)$. 

Moreover, we have mathematically proved that, under \textit{Justitia} an agent can be guaranteed to complete within a constant delay after its completion under GPS (the virtual fair scheduler). 
Due to the space limitation, please refer to Appendix~\ref{subsec:analysis} for the detailed theorem and proof.

\section{Evaluation}
\label{sec:eval}


\subsection{Setup}
\label{subsec:setup}

\begin{figure*}
    \centering
    \subfloat[LLaMA-7B]{
        \label{fig:e2e_7B}
        \centering
        \includegraphics[width=0.31\linewidth]{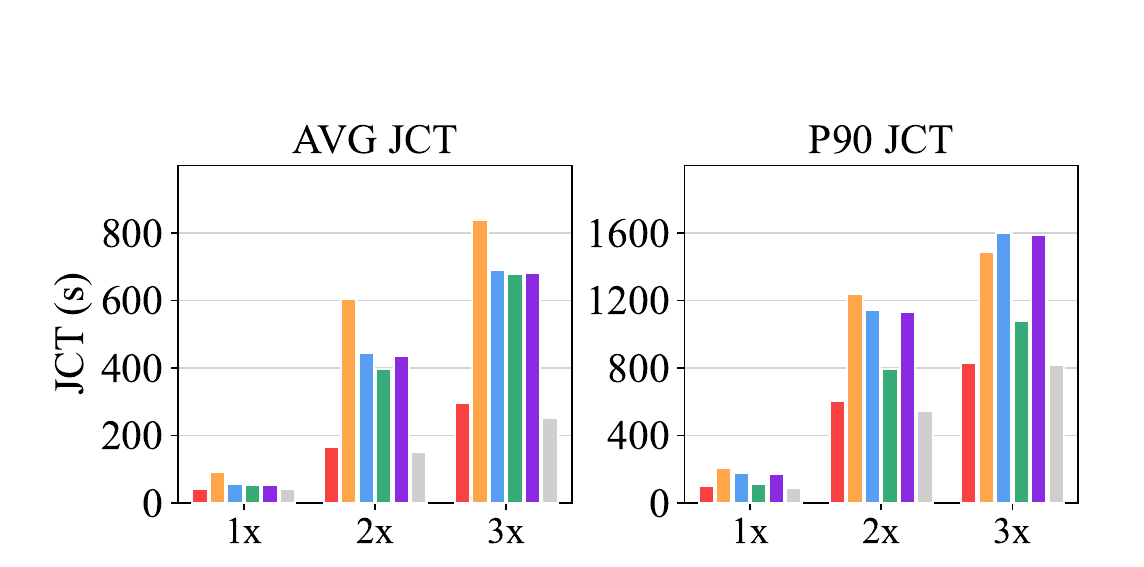}
    }
    \hfill
    \subfloat[LLaMA-13B]{
        \label{fig:e2e_13B}
        \centering
        \includegraphics[width=0.31\linewidth]{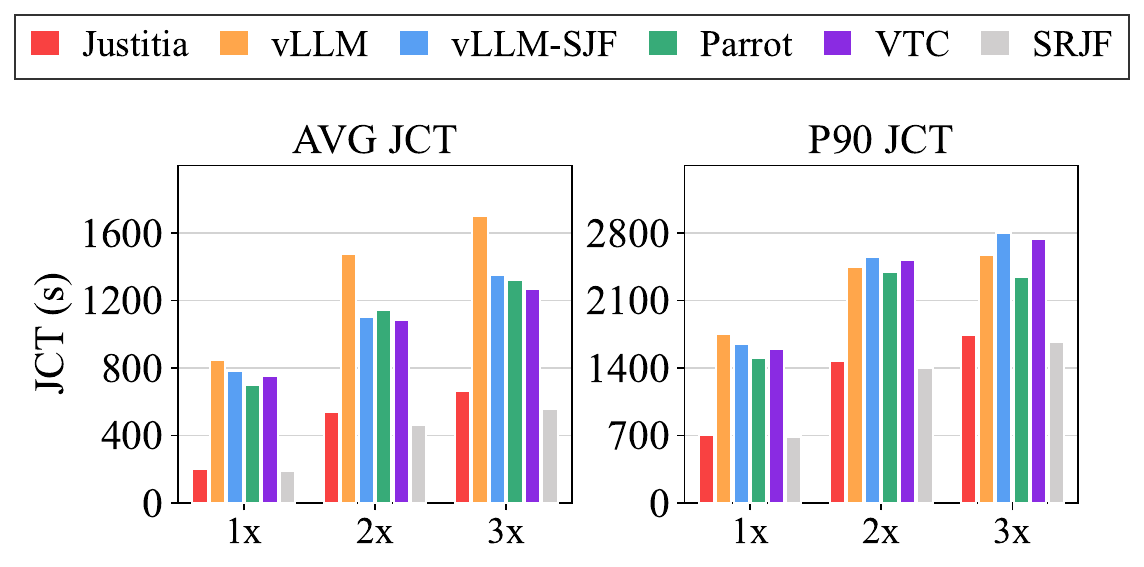}
    }
    \hfill
    \subfloat[Qwen2.5-32B]{
        \label{fig:e2e_32B}
        \centering
        \includegraphics[width=0.31\linewidth]{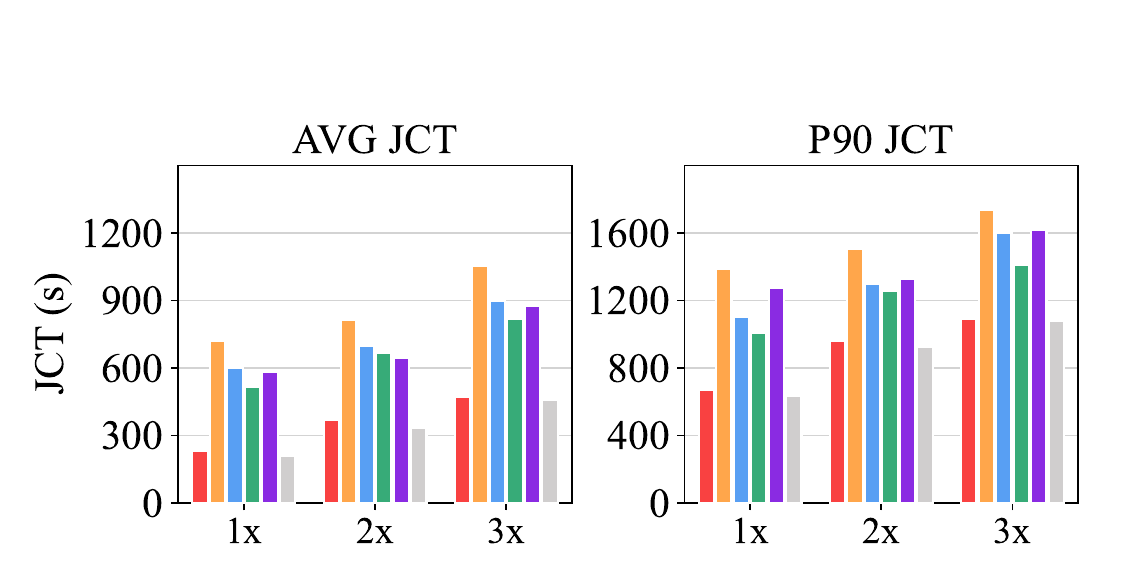}
    }
    \caption{JCT performance with different LLM backends.} 
    \label{fig:e2e_eff}
\end{figure*}



\noindent\textbf{Hardware Platform.} 
We have implemented \textit{Justitia} atop vLLM~\cite{vllm}, the mainstream LLM serving framework.
Due to the space limitation, the implementation details are placed in Appendix~\ref{sec:implementation}.
In testbed experiments, we evaluate \textit{Justitia} with diverse LLMs. 
The LLMs we use are LLaMA-7B (served on an A100-PCIe-40GB GPU), LLaMA-13B~\cite{llama2} (served on four V100-PCIe-16GB GPUs in tensor-parallel mode), and Qwen2.5-32B~\cite{qwen3_32b} (served on an H800-PCIe-80GB GPU).

\phm{Workloads.}
For our experiment, we created a mixed workload suite with 300 LLM agents, each with distinct inputs from the original datasets. 
To be specific, we include the following agent classes:  (a) MapReduce Summarization (MRS)~\cite{Langchain_MapReduce}, (b) Plan-and-Execution (PE)~\cite{shen2023hugginggpt}, (c) Code Checking (CC)~\cite{chern2023factool}, (d) Knowledge-Based-Query-Answering Verification (KBQAV)~\cite{chern2023factool},  (e) Equation Verification (EV)~\cite{chern2023factool}, (f) Fact Verification (FV)~\cite{yao2023react}, (g) ALFWorld Interaction (ALFWI)~\cite{yao2023react}. (h) Document Merging (DM)~\cite{besta2024graph} and (i) Self Consistency (SC)~\cite{wang2022self}.
Similar to prior work~\cite{qiao2021pollux,jayaram2023sia}, we set the sampling probability of small (EV, FV, CC, ALFWI and KBQAV---usually less than 1 min), medium (CG, PE and SC---usually between 1 and 10 min) and large (DM and MRS---usually longer than 10 min) agents to be 72\%, 26\%, and 2\%, respectively. 
Regarding agent submission, we follow the request arrival time in the production traces released by Mooncake~\cite{qin2024mooncake}, with the submission window respectively set to 6, 9, 18 mins (i.e., with workload intensity respectively be 3$\times$, 2$\times$ and 1$\times$).

\begin{figure}
	\centering
	\includegraphics[width=0.45\textwidth]{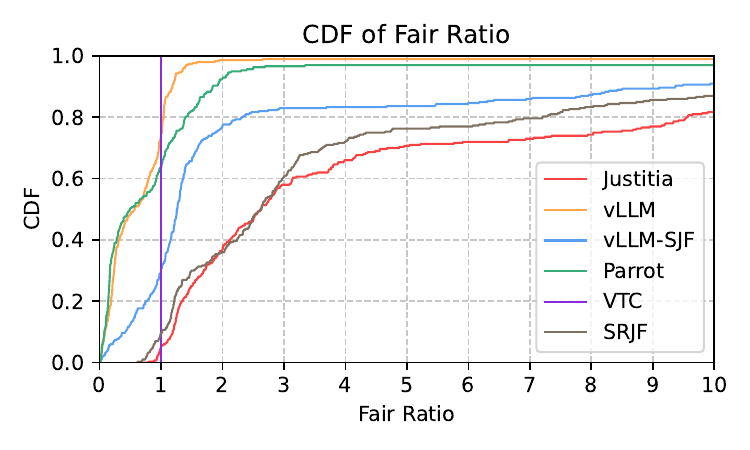}
	\caption{CDF of agents' finish-time fair ratios (i.e., realistic JCT normalized by VTC-JCT) under 3$\times$ density.}
    \label{fig:overall_fair}
\end{figure}


\phm{Baselines.} 
We compare \textit{Justitia} with five baselines: (a) vLLM~\cite{vllm}, which adopts FCFS policy at the inference level; 
(b) vLLM-SJF~\cite{shahout2024don}, which schedules LLM inferences following SJF policy with the Distillbert-predicted inference durations;
(c) Parrot~\cite{lin2024parrot}, which adopts the FCFS policy at the agent level; (d) VTC~\cite{vtc}, which seeks to approximate the instantaneous fair-sharing policy at the agent level; (e) SRJF, which uses our predicted serving cost to enforce the shortest-remaining-job-first policy at the \emph{agent} level.

\phm{Metrics.}
\textit{Justitia} is expected to behave well in both efficiency and fairness. 
Regarding efficiency, we adopt the average and P90 job completion time (JCT, meaning the duration between job arrival and completion); here a \emph{job} refers to a running agent triggered by a given user input. 
Regarding fairness, we adopt the notion of \emph{finish-time fair ratio}, which is the relative ratio between a job's realistic completion time and its completion time under a fair scheduler (we use VTC as the baseline).
The higher that ratio, the better the fairness level.





\subsection{End-to-end Scheduling Performance}
\label{subsec:overall}

\phm{Efficiency performance.}
We first evaluate the efficiency performance of \textit{Justitia}.
Fig.~\ref{fig:e2e_eff} shows the overall serving performance with different models.
From the two figures, \textit{Justitia} can substantially outperform the mainstream schedulers. 
Specifically, the average JCT under \textit{Justitia} is 57.5\% (61.1\%) better than that under VTC (Parrot).
In the meantime, \textit{Justitia} attains a very close JCT performance with SRJF,
indicating that it can attain near-optimal efficiency.

\phm{Fairness performance.}
Regarding the fairness performance, we further depict the \emph{cumulative distribution function} (CDF) of all the agents' finish-time fair ratio, with the results shown in Fig.~\ref{fig:overall_fair}.
It shows that 92\% agents can complete under \textit{Justitia} no later than it would have under VTC (with the worst-case delay be 26.0\%, which is much smaller than others).
This confirms the fairness superiority of \textit{Justitia}.
Interestingly, we find that SRJF can also attain a relatively good fairness performance---because it can avoid head-of-line blocking and, many prioritized agents cannot saturate the service backend---thus the remaining KV resources are multiplexed by others, exhibiting a modest fair-sharing effect.
Yet, SRJF may starve the large agents, which we next verify with a micro-benchmark experiment.

\subsection{Micro-benchmark Experiments}
\label{subsec:eval_micro}

\phm{Starvation avoidance under \textit{Justitia}.}
A well-known deficiency of SRJF is the starvation problem: when short agents keep arriving, the execution of long agents may be infinitely delayed. 
To confirm that, we conduct a micro-benchmark experiment. 
To be specific, we first submit a large ``elephant'' agent---MapReduce Summarization, and then keep submitting a set of ``mice'' agents---randomly from KBQAV, CC, ALFWI---once per second.
In Fig.~\ref{fig:micro_starvation}, we show the relationship between the JCT of the elephant agent and the number of ``mice'' agents respectively under SRJF and \textit{Justitia}. 
It clearly demonstrates that, with an increasing number of mice agents, the elephant agent may potentially be delayed forever under SRJF; in contrast, that delay under \textit{Justitia} is bounded regardless of the number of competing agents. 

\phm{Performance robustness against prediction errors.}
In reality, it is still possible that the predicted cost of an LLM agent is highly inaccurate. 
To check the performance robustness of \textit{Justitia} against prediction errors, we resort to another experiment with controlled prediction errors.
Given the profiled cost ground-truth (\emph{temperature} set to 0 to ensure recurrence), we scale it with randomly factors at different levels. 
As in Fig.~\ref{fig:error_robustness}, $\lambda\times$ means that the original cost information is scaled by a random factor in [$1/\lambda$, $\lambda$] before being used by \textit{Justitia} (when $\lambda$ is 1 we directly use the ground-truth).
Fig.~\ref{fig:error_robustness} shows that the performance of \emph{Justitia} is highly robust: compared with the ideal case ($\lambda=1$), the average JCT is inflated by only $9.5\%$ when $\lambda$ is set to $3$.

\subsection{Ablation Study}
\label{subsec:ablation}
In Sec.~\ref{subsec:modeling} and Sec.~\ref{subsec:prediction} we have respectively adopted \emph{memory-centric cost modeling} and \emph{MLP-based demand prediction}; here we check their effectiveness.

\begin{figure}
	\begin{minipage}[t]{0.45\linewidth}
	\centering
    	\includegraphics[width=0.9\linewidth]{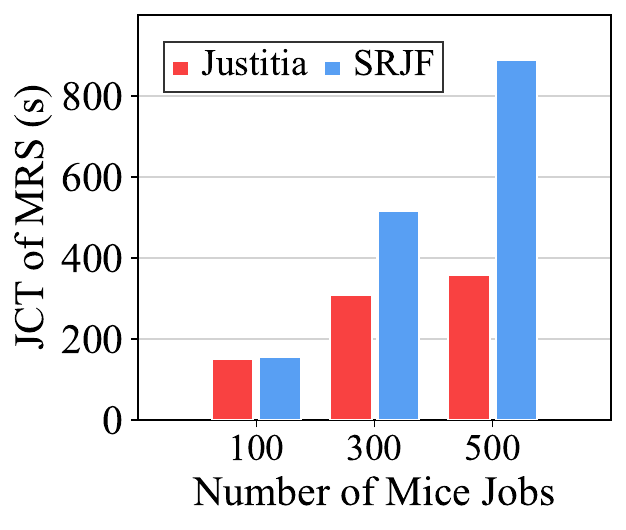}
    	\caption{Behaviors when serving a large agent with many small ones.}
    	\label{fig:micro_starvation}
	\end{minipage}%
	\hfill
	\begin{minipage}[t]{0.45\linewidth}
        	\includegraphics[width=0.9\linewidth]{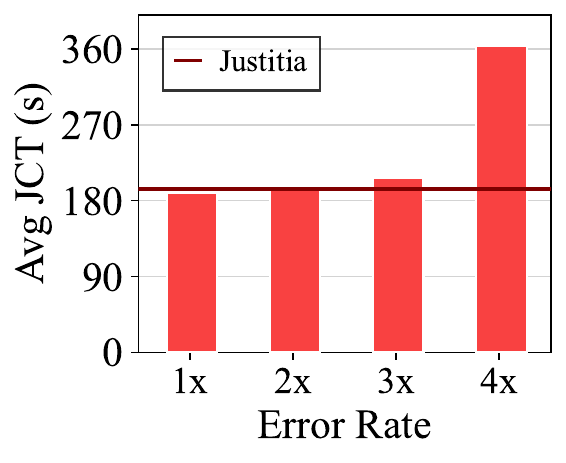}
           \caption{Justitia performance when prediction error amplifies.} 
    \label{fig:error_robustness}
	\end{minipage}
    \hfill
	\begin{minipage}[t]{0.45\linewidth}
    \includegraphics[width=0.9\linewidth]{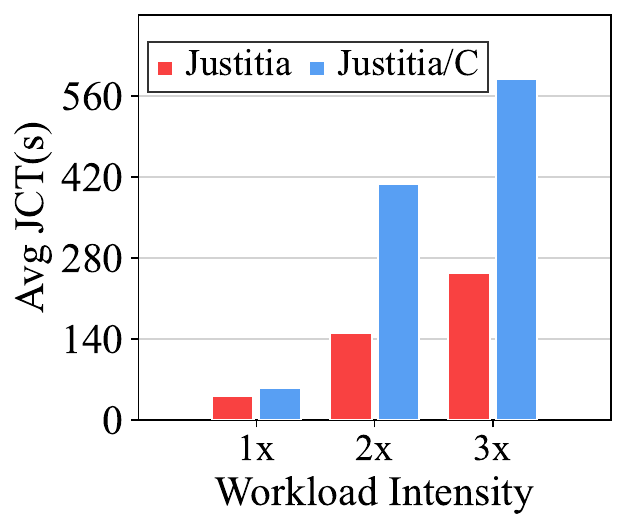}
           \caption{Performance comparison between different cost modeling methods.} 
    \label{fig:compare_compute-centric}
	\end{minipage}
    \hfill
	\begin{minipage}[t]{0.45\linewidth}
            \includegraphics[width=0.9\linewidth]{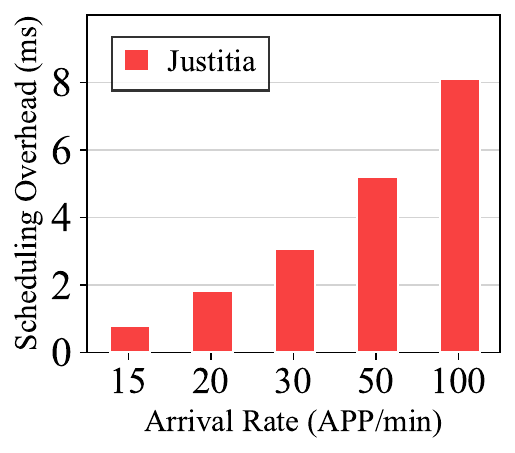}
                       \caption{Justitia scheduling delay at different request arrival rates.}
                \label{fig:scalability}
	\end{minipage}
\end{figure}


\phm{Necessity of memory-centric cost modeling.}
To check the necessity of memory-centric cost modeling, we introduce a variant of \textit{Justitia} by replacing its cost model with that of VTC (i.e., $p+2d$~\cite{vtc}), which we call \textit{Justitia}/C. 
We repeat the experiment in Fig.~\ref{fig:e2e_7B}, and compare the average and P90 JCT under \textit{Justitia}/C with the vanilla \textit{Justitia}. 
As shown in Fig.~\ref{fig:compare_compute-centric}, compute-centric cost modeling would incur a JCT performance degradation of up to 42.3\%, which confirms the necessity to adopt memory-centric cost modeling.

\phm{Effectiveness of MLP-based demand prediction method.}
To check the superiority of MLP-based demand prediction, we introduce another \textit{Justitia} variant: \textit{Justitia}-S3, which uses a Distillbert model for cost prediction in \textit{Justitia}.
In \textit{Justitia}-S3, the Distillbert model is trained on the same dataset as the MLP-based models.
In Table~\ref{table:pred} we compare the two methods in multiple aspects: average prediction error (prediction gap normalized by the ground-truth), average prediction overhead, average JCT, as well as the model training time.
Table~\ref{table:pred} suggests that our MLP-based method can remarkably outperform the distillbert-based method. 



\begin{table}
\caption{Performance comparison between MLP and Distillbert in prediction (under 2$\times$ workload density). MLP and Distillbert model training is conducted on a V100 GPU.} 
\label{table:pred}
\centering
\small
\begin{footnotesize}
\setlength{\tabcolsep}{4pt}  
\begin{tabular}{@{}>{\centering}p{1.5cm}cccc@{}}
\toprule
\textbf{\makecell{Prediction \\Model}} & \textbf{\makecell{Average \\ Relative \\ Error (\%)}} & \textbf{\makecell{Average \\ Inference\\ Overhead (ms)}} & \textbf{\makecell{Average \\JCT (s)}} & \textbf{\makecell{Training\\ Time}}\\ \midrule
MLP & 53.0 & 2.16 & 151.1 & $\sim$1 min\\
Distillbert & 452 & 55.7 & 366.7 & $\sim$2 h\\ \bottomrule
\end{tabular}
\end{footnotesize}
\end{table}


\subsection{Overhead Analysis}




The overhead in \textit{Justitia} stems from two sources: (1) the one-shot prediction performed when an agent arrives, and (2) the updating of system virtual time and the selection of highest-priority agent upon the arrival or completion of any agent. 
As can be seen also in Table~\ref{table:pred}, the prediction overhead for a single new agent is approximately 2.16 ms, which is a negligible cost. 
Fig.~\ref{fig:scalability} further demonstrates the average scheduling latency of \textit{Justitia} under varying arrival rates (reflecting different scheduling scales). 
The results confirm that the scheduling overhead remains consistently low (under 10 ms) in all scenarios.









\section{Conclusion}
\label{sec:conlucsion}


In this paper, we propose \textit{Justitia}, an efficient and also fair scheduler for task-parallel LLM agents. 
\textit{Justitia} works by scheduling LLM agents in a selective pampering manner, following their fair completion order. 
Specifically, it quantifies the true service cost of LLM agents in a memory-centric manner, and adopts a MLP-based method to predict the agent service cost a priori. 
Moreover, \textit{Justitia} adopts the fair queuing method to efficiently get the expected agent completion order under an idealized fair scheduler.
Testbed measurements with diverse workloads confirm that \textit{Justitia} can behave well in both fairness and efficiency. 

\begingroup
\raggedright
\bibliography{main}
\bibliographystyle{icml2026}
\endgroup

\newpage
\appendix
\onecolumn

\section{Agent-specific Demand Stability}
\label{appendix:agent_demand_stability}


In this section, we elaborate that the resource demands of an LLM agent is relatively stable across different trial runs (each with a distinct agent input). 
In fact, since each LLM agent is designed for specific use cases, the resource demands of which are relatively stable across different execution runs. 
Fig.~\ref{fig:app_demand} shows the execution information (input/output length) of specific inference stages in the two LLM agent over 100 trial runs, which exhibits strong demand correlations.  
For example, \texttt{generate-queries} inferences in Fact-Verification agent all have an input token length between 360 and 380. 

\begin{figure}[!th]
\centering
\setlength{\tabcolsep}{0pt} 
\begin{tabular}{cc}
	\parbox[b]{0.28\linewidth}{\centering
		\includegraphics[width=\linewidth]{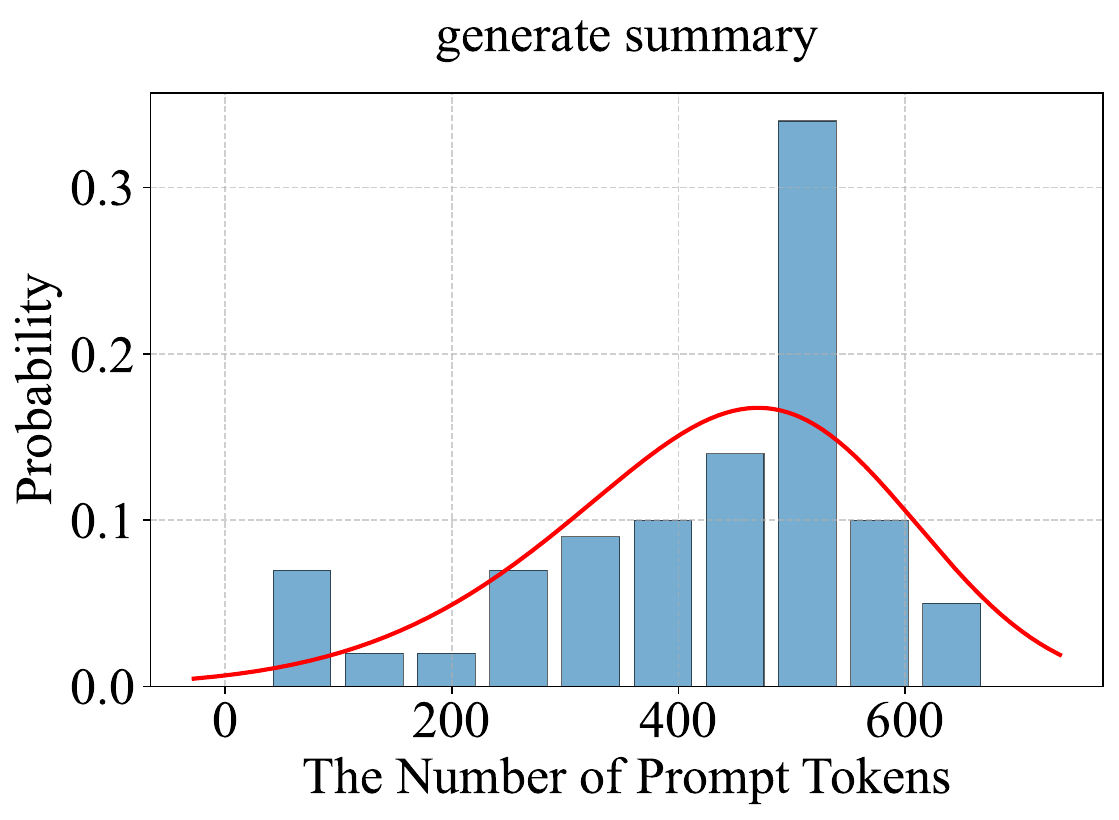}
	} & 
	\parbox[b]{0.28\linewidth}{\centering
		\includegraphics[width=\linewidth]{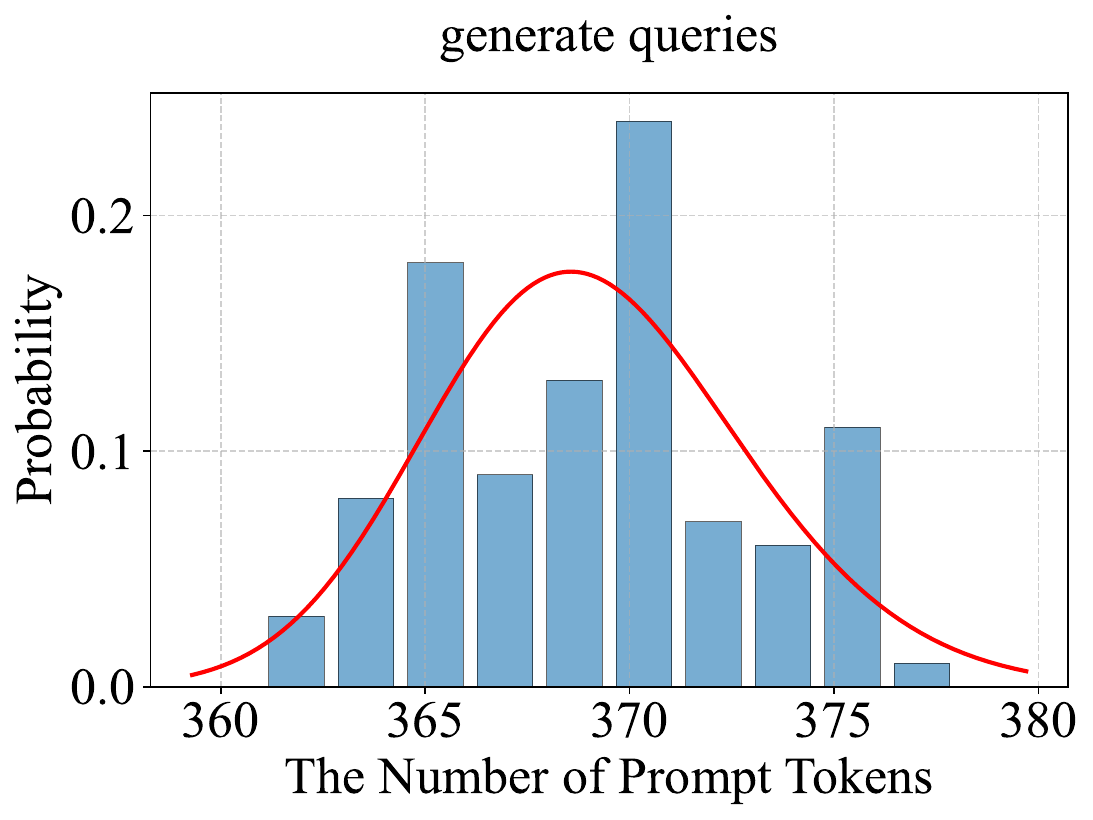}
	} \\
	\parbox[b]{0.28\linewidth}{\centering
		\includegraphics[width=\linewidth]{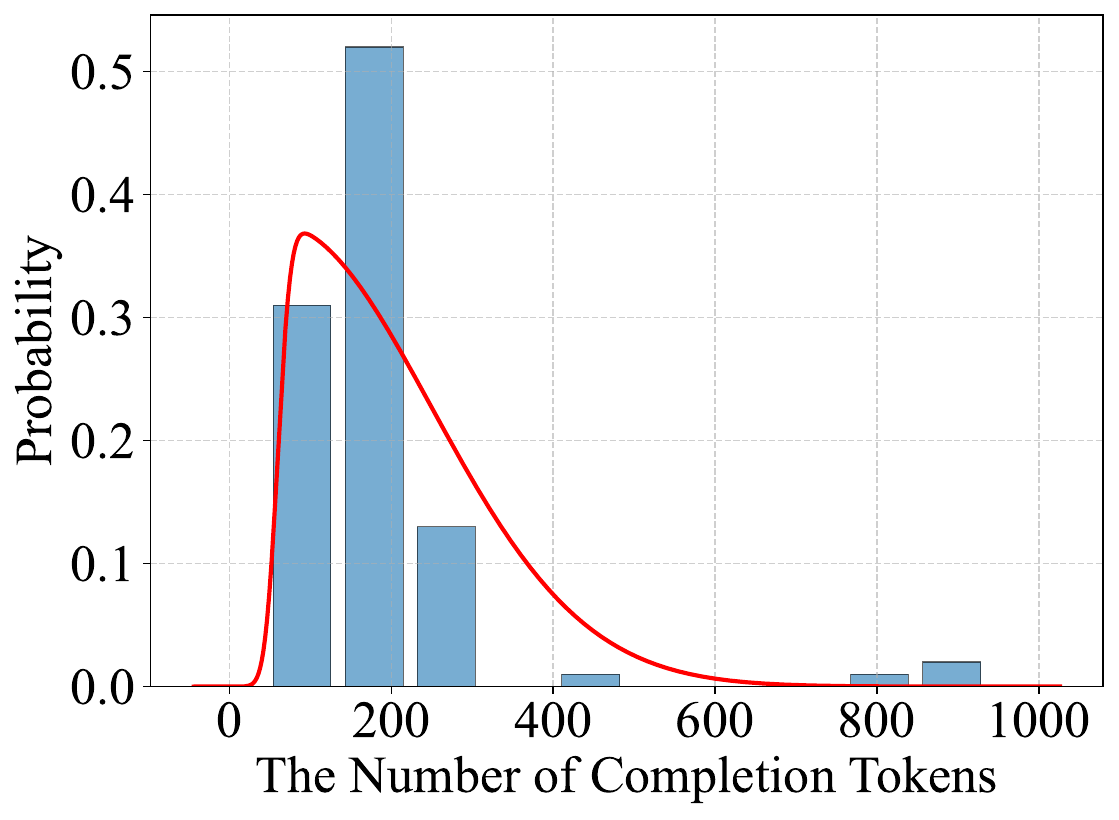}
		\vspace{-0.1in}
		\caption*{(a) Summarization}
	} &
	\parbox[b]{0.28\linewidth}{\centering
		\includegraphics[width=\linewidth]{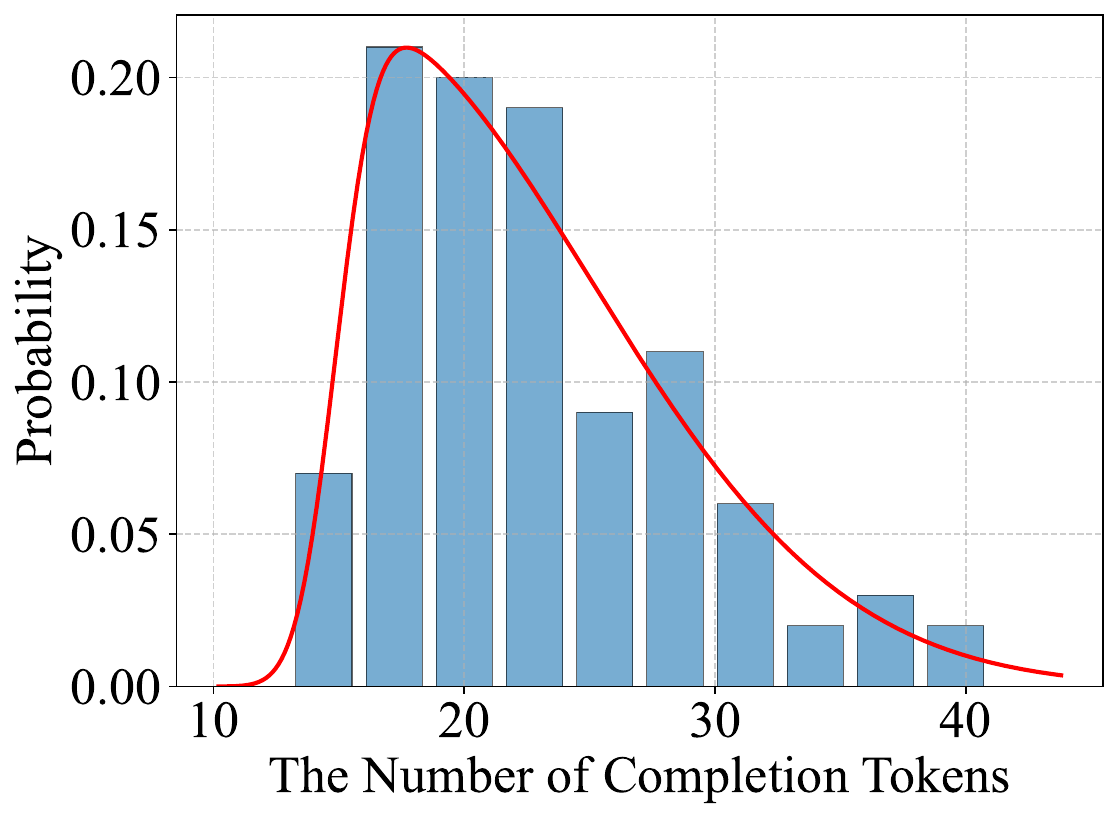}
		\vspace{-0.1in}
		\caption*{(b) Fact Verification}
	}
\end{tabular}
\vspace{-0.1in}
\caption{Prompt and decoding length distribution respectively for the \texttt{generate-summary} and \texttt{generate-queries} inferences in the MapReduce-Summarization agent and the Fact-Verification agent. In each case, we divide the length range into 10 buckets, and calculate the value appearance probability in each bucket (accompanied by the fitted curves assuming skewed Gaussian distribution).}
\label{fig:app_demand}
\vspace{-0.1in}
\end{figure}

\section{Delay Analysis}
\label{subsec:analysis}

In this part we analyze the fairness properties of \textit{Justitia}.
In practice, an inference cannot be started if the available KV block volume is smaller than its prompt size, which may cause the memory fragmentation problem.
For ease of analysis, we assume the prompt size of typical inferences are much smaller in scale than the total KV cache space.
Our measurement show that, in the released Mooncache dataset~\cite{qin2024mooncake}, each request in average only takes 3.20\% of the total KV cache size when running LLaMA-3.3 70B on H100 (with 28.65GB KV cache space).
Therefore, we can faithfully neglect the fragmentation problem.
That is, if there are inferences waiting in the queue, all the KV-blocks on the server would be fully utilized.


\phm{Symbols.}
We let agents be indexed in ascending order of their inference start time.
That is, app-$j$ is the $j$\textsuperscript{th} agent that is allocated KV-blocks in \textit{Justitia}. 
Meanwhile, we let $C_j$ represent the total KV token-time of app-$j$, which is the sum of all the KV token-time of all its inferences. 
We also let $C_{max}$ denote the maximum total KV token-time among all agents. 
Similarly, we let $c_{max}$ represent the maximum KV token-time required by any single inference across all agents. 
Finally, we let $f_j$ denote the completion time of app-$j$ in \textit{Justitia}, and $\bar{f}_j$ the GPS completion time. 
Table~\ref{tbl:notations} summarizes the notations used in the analysis.
Given these definitions, we then establish the constant delay bound of \textit{Justitia} through the following theorem.

\newcommand{\specialcell}[2][c]{%
\begin{tabular}[#1]{@{}c@{}}#2\end{tabular}}
\begin{table}[!th]
   \centering
   \renewcommand{\arraystretch}{1.1}
   \footnotesize
   \caption{Summary of important notations and definitions.}
   \begin{tabular}{|c|l|}
     \hline
     $M$ & The number of KV-blocks in a server \\
     $\bar{f}_j$ & The agent completion time for app-$j$ in GPS \\
     $f_j$ & The agent completion time for app-$j$ in \textit{Justitia} \\
     $C_j$ & The total KV token-time required by app-$j$ \\
     $C_\mathrm{max}$ &  The maximum KV token-time required by any agent \\
     $c_\mathrm{max}$ & The  maximum KV token-time required
by any single inference\\
     $a_j$ & The arrival time of app-$j$\\
     $e_j$ & The ending time of the slowdown period of app-$j$\\
     \hline
   \end{tabular}
   \label{tbl:notations}
\end{table}

\begin{theorem}[Constant delay bound]
	\label{thm:const_delay_bound}
	With \textit{Justitia}, an agent is guaranteed to complete within a constant time after its completion in GPS, i.e., for each app-$j$, we have
	\begin{equation}
		\label{eq:slowdown_bound}
		f_j - \bar{f}_j \leq 2c_\mathrm{max} + C_\mathrm{max} / M.
	\end{equation}
\end{theorem}

\begin{proof}

Our analysis critically focuses on the \emph{slowdown period} of an agent. In particular, we say an agent is \emph{slowed down} at time $t$ if it has a
\emph{backlogged} inference waiting for service. In other words, at any moment in the slowdown period, the agent could have run more inferences if receiving more KV-blocks. Because slowdown delays the agent completion, bounding the timespan of the slowdown periods is the key to analyzing the longest possible delay.

For each app-$j$, we let $a_j$ be the arrival time of app-$j$. 
Depending on the number of available KV-blocks at time $a_j$, app-$j$ is either slowed down, or allocated a sufficient number of KV-blocks to run all inferences right after the arrival. 
In particular, we let $e_j$ be the time when the slowdown period of app-$j$ ends. 
We then have $a_j < e_j$ if app-$j$ experiences slowdown, and $a_j = e_j$ otherwise. 
In either case, after the slowdown period, app-$j$ runs all the backlogged inferences in parallel, and is guaranteed to complete after at most the maximal inference runtime, i.e.,
	
\begin{equation}
\label{eq:comp_time_bound}
f_j \le e_j + c_\mathrm{max}.
\end{equation}

Therefore, to bound the agent completion time, it is critical to analyze when the slowdown period ends (i.e., $e_j$). 

We consider the most general case where app-$j$ is slowed down right after the arrival, i.e., $a_j < e_j$.
During the slowdown period $[a_j, e_j]$, all the KV-blocks are busy. 
While \textit{Justitia} preferentially offers KV-blocks to agents in ascending order of their GPS completion times, agents may start services out of order due to the dynamic arrivals. 
In particular, an agents that completes before app-$j$ in GPS may arrive late, after app-$j$ starts in \textit{Justitia}. 
Let $\mathcal{A}$ be the set of all these agents, i.e., $\mathcal{A} = \{k \mid k > j \mbox{ and }
	\bar{f}_k \le \bar{f}_j \}$. 
    
On the other hand, an agent that completes after app-$j$ in GPS may start earlier in \textit{Justitia}, before app-$j$ arrives. Let app-$m$ be such an agent that is serviced \emph{the most recently}. That is, $m$  is the \emph{largest} integer satisfying both $m \le j-1$ and $\bar{f}_m > \bar{f}_j$, i.e., $\bar{f}_m > \bar{f}_j \ge f_k$ for all $m < k < j$.
In other words, app-$m$ completes after agents $m+1, \dots, j$ in GPS, but is allocated KV-blocks before all these agents in \textit{Justitia}. 
We let $\mathcal{B}$ be the set of all these agents that complete after app-$j$ in GPS but start earlier in \textit{Justitia}, i.e., $\mathcal{B} = \{ k \mid m < k \le j \}$.

By definition, agents in $\mathcal{A}$ and $\mathcal{B}$,
though completing earlier than app-$m$ in GPS, 
	are serviced no earlier than app-$m$ in \textit{Justitia}.
	These agents must have not yet arrived before app-$m$ is allocated KV-blocks---otherwise \textit{Justitia} would have
	serviced them before app-$m$. We then have
	\begin{equation}
		\textstyle
		\min_{k \in \mathcal{A} \cup \mathcal{B}}{\{ a_k \}} \ge b_m,
		\label{eq:m_arrival_time}
	\end{equation}
	where $b_m$ is the first time when app-$m$ is allocated KV-blocks in \textit{Justitia}.
	This suggests that since $b_m$, GPS has completely serviced,
	\emph{at least}, all agents in $\mathcal{A} \cup \mathcal{B}$ by the time app-$m$ finishes, i.e.,
	\begin{equation}
		\textstyle
		\bar{f}_j \ge b_m + \frac{1}{M} \sum_{k \in \mathcal{A} \cup \mathcal{B}}{C_k}.
		\label{eq:GPS_comp_time_case_2-2}
	\end{equation}
	
We next analyze $f_j$, the completion time of app-$m$ in \textit{Justitia}.
During time interval $[b_m, e_j]$, \textit{Justitia} may have completed all the agents in $\mathcal{A}$ and $\mathcal{B}$, along with at most $M - 1$ inferences requiring a maximal runtime. 
In addition, app-$m$ is allocated KV-blocks at $b_m$, and may also complete before $e_j$. 
Finishing these works using all $M$ KV-blocks takes at most
\begin{equation}
	\label{eq:e_j_bound_case_2-2}
	  \textstyle
	  e_j \le b_m + \frac{1}{M} [ \sum_{k \in \mathcal{A} \cup \mathcal{B}} C_k + C_m 
	  + (M - 1)c_\mathrm{max} ].
	\end{equation}
	Plugging \eqref{eq:e_j_bound_case_2-2} into \eqref{eq:comp_time_bound} and
	subtracting \eqref{eq:GPS_comp_time_case_2-2}, we have
	\begin{equation*}
	  \textstyle
	  f_j - \bar{f}_j \le c_\mathrm{max} + \frac{M-1}{M} c_\mathrm{max} + C_m / M
	  < 2 c_\mathrm{max} + C_\mathrm{max} / M.
	\end{equation*}

This completes the proof. 
\end{proof}

\section{Implementation Details}
\label{sec:implementation}

We have implemented \textit{Justitia} atop vLLM~\cite{vllm}, the mainstream LLM serving framework.
The overall workflow with \textit{Justitia} is shown in Fig.~\ref{fig:workflow}.

\begin{figure}[!th]
    \centering
    \includegraphics[width=0.52\textwidth]{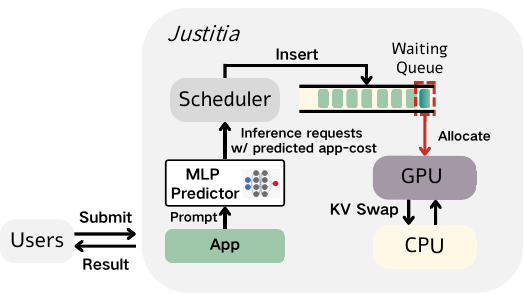} 
    \caption{\textit{Justitia} Workflow} 
    \label{fig:workflow}
\end{figure}

As elaborated in Fig.~\ref{fig:workflow}, each time a user submits an application request, the application-specific (MLP) predictor would first predict the application-level service cost of that application based on user input in the ongoing run (following the method described in Sec.~\ref{subsec:prediction}).
The \textit{Justitia} scheduler would then calculate the virtual completion time of that application---based on its predicted service cost and the current virtual time (as described in Sec.~\ref{subsec:queuing}).
That application-level virtual completion time would be used as the scheduling priority of its inference tasks when inserting them into the task waiting queue. 
For each scheduled inference task, we serve it following the standard vLLM logic.
Each time the KV cache space is full, \textit{Justitia} moves \emph{some} tasks to the swapped queue, with their KV contents temporally swapped to CPU memory; meanwhile, no new tasks can be admitted to the running queue when there are pending tasks in the swapped queue.

\section{Additional Related Works}

Apart from the scheduling works in Sec.~\ref{sec:closely_related}, many additional works seek to accelerate individual LLM requests in a series of aspects.
For example, FlashAttention~\cite{dao2022flashattention} and FlashDecoding~\cite{hong2024flashdecoding}  algorithms seek to accelerate LLM inference at the operator level;
model-quantization~\cite{kim2023squeezellm} and prefill-decode-separation~\cite{zhong_distserve_2024} methods have been proposed to maximize backend utilization;
Speculative decoding methods like \cite{cai2024medusa} are also applied to improve the inference throughput by accelerating the decoding process.

Meanwhile, agent serving is also a popular research topic in the recent literature. 
Apart from the task-parallel nature which is the focus of this paper, modern agents may also stand out by their tool usage capabilities,
as exemplified by \emph{Toolformer}~\cite{schick2023toolformer} (API calling for arithmetic/fact-checking), \emph{Gorilla}~\cite{patil2024gorilla} (large-scale API libraries), and \emph{ToolLLM}~\cite{qin2023toolllm} (multi-tool reasoning frameworks).
For those tool-usage agents, some other optimization techniques have also been proposed~\cite{abhyankarinfercept,zhai2026toolcaching}.
It can be expected that both task-parallel nature and tool-diversifying nature would be common in future agent design, thus their optimization techniques can be jointly applied.

\end{document}